\documentclass[11pt]{article}
\usepackage{amssymb, latexsym, mathtools, amsthm}
\begingroup
    \makeatletter
    \@for\theoremstyle:=definition,remark,plain\do{%
        \expandafter\g@addto@macro\csname th@\theoremstyle\endcsname{%
            \addtolength\thm@preskip\parskip
            }%
        }
\endgroup
\usepackage{enumerate}
\usepackage{pgf,tikz}
\usepackage{pdfsync}
\usepackage{txfonts}
\usepackage[T1]{fontenc}
\usepackage{booktabs}
 \usepackage{graphicx}
\definecolor{dnrbl}{rgb}{0,0,0.3}
\definecolor{dnrgr}{rgb}{0,0.3,0}
\definecolor{dnrre}{rgb}{0.5,0,0}
\usepackage[colorlinks=true, citecolor=dnrgr, linkcolor=dnrre, urlcolor=dnrre] {hyperref}
\usepackage{xcolor}
\usepackage{parskip}
\usepackage{tabularx, colortbl}
\usepackage{geometry}
\usepackage{etoolbox}
 \usepackage[scriptsize, up]{caption}
\usepackage{pgf,tikz}
\usetikzlibrary{decorations, decorations.pathmorphing}
\usetikzlibrary {shapes}
\theoremstyle{plain}
\newtheorem{thm}{Theorem}[section]

\newtheorem*{conjecture}{Conjecture}
\newtheorem{lem}[thm]{Lemma}

\theoremstyle{definition}

\newtheorem{defi}[thm]{Definition}
\makeatletter

\let\c@table\c@figure
\makeatother

\newcommand{\Nat}{\mathbb{N}}

\newcommand{\restr}{\upharpoonright}  
\DeclarePairedDelimiter{\extp}{\mathtt{EXT}(}{)} 
\newcommand{\un}{\uparrow} 
\newcommand{\de}{\downarrow} 

\DeclarePairedDelimiter{\dbra}{\llbracket}{\rrbracket}

\newcommand{\LL}{\mathcal{L}}

\newcommand{\MM}{\mathcal{M}} 
\newcommand{\CC}{\mathcal{C}}

\newcommand{\PP}{\mathcal{P}}

\newcommand{\ml}{Martin-L\"{o}f }

\newcommand{\eg}{e.g.\ }
\newcommand{\ie}{i.e.\ }
\newcommand{\ce}{c.e.\ }

\newcommand{\pf}{prefix-free }

 \newcommand{\VV}{\mathcal V}
 \newcommand{\GG}{\mathcal G} 

\renewenvironment{abstract}
 { \normalsize
  \list{}{
    \setlength{\leftmargin}{.0cm}%
    \setlength{\rightmargin}{\leftmargin}%
    }%
  \item {\bf \abstractname.} \relax}
 {\endlist}
\usepackage{titlesec}
 \titleformat{\paragraph}
{\normalfont\normalsize\em\bfseries}{\theparagraph}{1em}{$\blacktriangleright$\hspace{0.2cm}}
\titlespacing*{\paragraph}
{0pt}{3.25ex plus 1ex minus .2ex}{0.5ex plus .2ex}

\title{Algorithmic learning of probability distributions from random data in the limit 
\thanks{Barmpalias was supported by the 
1000 Talents Program for Young Scholars from the Chinese Government No.~D1101130.
Partial support was also received from the China Basic Research Program No.~2014CB340302.
Thanks to Fang Nan for comments on an earlier draft.}}

\author{George Barmpalias  \and Frank Stephan}
\date{\today}
\begin{document}
\maketitle
\begin{abstract}
We study the problem of identifying a probability distribution for some 
given randomly sampled data in the limit, in the context of algorithmic learning theory
as proposed recently by Vitanyi and Chater  \cite{VITANYI201713}.
We show that there exists a computable partial learner for the computable 
probability measures, while
by Bienvenu, Monin and Shen \cite{Bienvenu2014} 
it is known that  there is no computable learner for the computable probability  measures.
Our main result is the characterization of the oracles that compute explanatory learners
for the computable (continuous) probability  measures as the high oracles.
This provides an analogue of a well-known result of Adleman and Blum \cite{AdlemanB91}
in the context of learning computable probability distributions.
We also discuss related learning notions such as behaviorally correct learning and other
variations of explanatory learning, in the context of learning probability distributions from data.
\end{abstract}
\vspace*{\fill}
\noindent{\bf George Barmpalias}\\[0.5em]
\noindent
State Key Lab of Computer Science, 
Institute of Software, Chinese Academy of Sciences, Beijing, China.\\[0.2em] 
\textit{E-mail:} \texttt{\textcolor{dnrgr}{barmpalias@gmail.com}}.
\textit{Web:} \texttt{\textcolor{dnrre}{http://barmpalias.net}}\par
\addvspace{\medskipamount}\medskip\medskip

\noindent{\bf Frank Stephan}\\[0.5em]  
Departments of Mathematics and Computer Science,
National University of Singapore,
2 Science Drive 2, Singapore 117543,
Republic of Singapore.\\[0.2em]
\textit{E-mail:} \texttt{\textcolor{dnrgr}{fstephan@comp.nus.edu.sg.}}
\textit{Web:} \texttt{\textcolor{dnrre}{http://www.comp.nus.edu.sg/$\sim$fstephan/}}
\vfill \thispagestyle{empty}
\clearpage

\section{Introduction}\label{A4yxNNSDkW}
We are interested in the following informally stated  general 
problem, which we study in the context of 
formal language identification and algorithmic learning theory:

\begin{equation}\label{v9CcOK3DX6}
\parbox{12cm}{
Given a probability distribution $\PP$ and a sufficiently large sample of
randomly chosen
data from the given distribution, learn or estimate
a probability distribution with respect to which the sample has been randomly sampled.
}
\end{equation}

Problem \eqref{v9CcOK3DX6} has a long history in statistics (\eg see \cite{Vapnik:1982:EDB})
and has more recently been approached in the context of computational learning,  
in particular the probably approximately correct (PAC) learning model, starting with
\cite{Kearns:1994:LDD}. The same problem was recently approached in the context of
 Algorithmic Learning Theory, in the tradition of Gold \cite{GOLD1967447},
 and Kolmogorov complexity
by Vitanyi and Chater in  \cite{VITANYI201713}.\footnote{Probabilistic methods and 
learning concepts in formal language and algorithmic learning theory 
have been studied long before \cite{VITANYI201713},
see \cite{Pitt:1989:PII} and the survey \cite{Ambainis:2001}. However most of this work focuses
on identifying classes of languages or functions using probabilistic strategies, rather than
identifying probability distributions as Problem \eqref{v9CcOK3DX6} asks.
Bienvenu and Monin \cite[Section IV]{Bienvenu:2012:VNB:235} do study a form of 
\eqref{v9CcOK3DX6} through a concept that they call {\em layerwise learnability} of 
probability measures
in the Cantor space,
but this is considerably different than 
Vitanyi and Chater in  \cite{VITANYI201713} and the concepts of
Gold \cite{GOLD1967447}, the most important difference being that
it refers to classes of probability 
measures that are not necessarily contained in the computable
probability measures.}

The learning concepts discussed in  Vitanyi and Chater  \cite{VITANYI201713} are
very similar in nature to the classic concepts of 
algorithmic learning
which are
motivated by the problem of language learning in the limit (see \cite{PINKER1979217})
but they differ in two major ways. 
In the classic setting, one starts with a class of languages or functions which have a finite
description (\eg they are computable) and the problem is to find an algorithm (often called
{\em a learner}) which can infer,
given a sufficiently long text from any language in the given class, or a sufficiently long
segment of the characteristic sequence of any function in the given class, a description
of the language or function in the form of a grammar  
or a program. More precisely, the desired algorithm makes
successive predictions given longer and longer segments of the input sequence, and is required
to converge to a correct grammar or program for the given infinite input.

If we apply the concept of identification in the limit to Problem  \eqref{v9CcOK3DX6}, according to
Vitanyi and Chater  \cite{VITANYI201713}, one starts with a class $\VV$ of finitely describable
probability distributions (say,  the computable measures on the Cantor space) and we have the following differences with respect to the classic setting:
\begin{itemize}
\item the inputs on which the learner is supposed to succeed in the limit are random
sequences with respect to some probability distribution in the given class $\VV$,
and not elements of $\VV$;
\item success of the learner $\LL$ on input $X$ 
means that $\LL(X\restr_n)$ 
converges, as $n\to\infty$, to
a description of some element of $\VV$ with respect to which $X$ is random.
\end{itemize}
First, note that just as in the context of computational learning theory, here too
we need to restrict the probability distributions in Problem \eqref{v9CcOK3DX6}
to a class of `feasible' distributions, which in our case means computable distributions 
in the Cantor space. Second, in order to specify the learning concept we have described,
we need to define what we mean by random inputs $X$ with respect to a computable 
distribution $\PP$ in
the given class $\VV$ on which the learner is asked to succeed. 
Vitanyi and Chater  \cite{VITANYI201713} ask the learner to succeed on every real $X$
which is {\em algorithmically random}, in the sense of \ml \cite{MR0223179}, with respect to some computable probability measure.\footnote{From this point on we will use the term 
(probability) {\em measure} instead of distribution, since the literature in algorithmic randomness
that we are going to use is mostly written in this terminology.} Then the interpretation of 
Problem  \eqref{v9CcOK3DX6} through the lenses of algorithmic learning theory
and in particular the ideas of Vitanyi and Chater  \cite{VITANYI201713} is as follows:
\begin{equation}\label{sGTC9dqWCS}
\parbox{12cm}{Given a computable measure $\mu$ and an
algorithmically random stream $X$ with respect to $\mu$,
learn in the limit (by reading the initial segments of $X$) a
computable measure $\mu'$ with respect to which $X$ is
algorithmically random.}
\end{equation}
This formulation invites many different formalizations of learning concepts
which are parallel to the classic theory of algorithmic learning\footnote{EX-learning, BC-learning,
BC$^{\ast}$-learning etc. In Odifreddi \cite[Chapter VII.5]{Odifreddi:99} the reader can find
a concise and accessible introduction to these basic learning concepts and results.},
and although we will comment on some of them later on, this article is
specifically concerned with EX-learning ({\em explanatory learning}, one of the main
concepts in Gold \cite{GOLD1967447}), which means
that in \eqref{sGTC9dqWCS} we require the learner to eventually converge to a specific
description of the computable measure\footnote{as opposed to, for example,
eventually giving different indices of the same  
measure, or even different measures all of which satisfy the required properties.}
 with the required properties.
 
Formally, a {\em learner} is a computable function $\LL$ from the set of binary strings $2^{<\omega}$ to
$2^{<\omega}$. 

\begin{defi}[Success of learners on measures]\label{pRSkquPLEN}
A learner $\LL$, $EX$-succeeds on a measure $\mu$ if for every $\mu$-random real $X$
the limit of $\LL(X\restr_n)$ as $n\to\infty$ exists and is an index of a 
computable measure $\nu$ such that
$X$ is $\nu$-random.
\end{defi}

Vitanyi and Chater  \cite{VITANYI201713} observed 
for any uniformly computable class $\CC$ of computable measures
there exists a computable learner which
is successful on all of them, in the sense that it correctly guesses appropriate
computable measures for every stream which is $\mu$-random with respect to some 
$\mu\in\CC$. Then Bienvenu, Monin and Shen \cite{Bienvenu2014}
 showed that the class of computable measures is not learnable in this way.
This result can be viewed as an analogue of the classic theorem in 
 Gold \cite{GOLD1967447} that the class of computable functions is not EX-learnable.
In relation to the latter, Adleman and Blum \cite{AdlemanB91} showed that
the the oracles that EX-learn all computable functions are exactly the oracles $A$
whose jump computes the jump of the halting problem ($\emptyset''\leq_T A'$), 
\ie the oracles that can decide in the limit
the totality of partial computable functions.

In this article we show that  an oracle $A$ can learn the class of computable measures
(in the sense that it computes the required learner) if and only if it is high, \ie $A'\geq_T \emptyset'$.
We prove the following form of this statement, taking into account the characterization of
high oracles from  Martin \cite{Martinhigh} as the ones that can compute a function which dominates
all computable functions.

\begin{thm}\label{UejsgfHHK}
If a function dominates all computable functions, then 
it computes a learner which $EX$-succeeds on all computable measures.
Conversely, if a learner $EX$-succeeds on all computable (continuous)  
measures then it computes
a function which dominates all computable functions.
\end{thm}

This provides an analogue of the result of Adleman and Blum \cite{AdlemanB91} in the context
of learning of probability measures, and is the first oracle result in this topic. In particular,
it shows that the computational power required for learning all computable reals
in the sense of Gold \cite{GOLD1967447} (identification by explanation)
is the same power that is required for learning all computable measures in
the framework of Vitanyi and Chater  \cite{VITANYI201713}.
Our methods
differ from  Bienvenu, Monin and Shen \cite{Bienvenu2014}, and borrow some ideas from  
Adleman and Blum \cite{AdlemanB91}, but in the context of sets of positive measure instead of
reals. Moreover our arguments show a stronger version of Theorem \ref{UejsgfHHK},
which we detail in Section \ref{of5zJimRlx}, and which roughly says that the theorem holds
also for fixed positive probability of $EX$-success on all measures.

On the positive side, 
Osherson, Stob and Weinstein \cite{STL1st} introduced the notion of partial learning
for computable sequences, and showed that there is a 
computable learner which partially learns all computable binary sequences.
We introduce the corresponding notion for measures and show an analogous result.

We say that a learner $\LL$ {\em partially succeeds} on measure $\mu$ if
for all $\mu$-random $X$ there exists a $j_0$ such that
\begin{itemize}
\item there are infinitely many $n$ with $\LL(X\restr_n)=j_0$;
\item if $j\neq j_0$ then there are only finitely  many $n$ with $\LL(X\restr_n)=j$;
\item $\mu_{j_0}$ is a computable measure such that $X$ is $\mu_{j_0}$-random.
\end{itemize}

\begin{thm}\label{8m2M82eMhm}
There exists a computable learner which partially succeeds on all computable measures.
\end{thm}
We give the proof of Theorem \ref{8m2M82eMhm} in Section \ref{r9nX2v56yb}.

Behaviorally correct learning, or BC-learning, is another standard notion
in algorithmic learning theory, and requires that for all computable $X$,
there exists some $n_0$ such that for all $n>n_0$ the learner on $X\restr_n$ predicts 
an index of a computable function with characteristic sequence $X$
(instead of converging to a single such index as in explanatory learning).
Bienvenu, Figueira, Monin and  Shen \cite{Bienvealgindpro} considered the
analogue of BC-learning for measures and showed that
there exists no computable learner which BC learns all computable measures.
They also considered the analogue of $BC^{\ast}$-learning (which is the same as BC but
ignoring finite differences of the functions) for measures and showed that
there exists a computable learner which $BC^{\ast}$-learns all computable measures, hence
giving an analogue of a theorem of Harrington  
who proved the same in the classical setting.\footnote{Harrington's result is
reported in \cite{CASE1983193}. The computable learner $\LL$ with the stated property, given
$\sigma$ outputs an index of the following program, where $(\varphi)$ is a standard list of all
partial computable functions:
on input $n$, search for the least $e\leq n$ such that $\varphi_e[n]$ extends $\sigma$;
if such exists, output $\varphi_e(n)$; otherwise output 0. It is not hard to see that this learner
has the required properties.} 
In Section \ref{of5zJimRlx} we discuss additional facts about EX and BC learning that one
may try to establish using the methods developed in the present paper.

\section{Background facts and the easier proofs}
Consider the Cantor space $2^{\omega}$, which is the set of all infinite binary sequences
which we call {\em reals}. This is a topological space generated 
by the basic open sets $\dbra{\sigma}=\{\sigma\ast X\ |\  X\in 2^{\omega}\}$
for all binary strings $\sigma$, where $\ast$ denotes concatenation.
Then the open sets can be represented by sets of strings $Q$ and we use
$\dbra{Q}$ to denote the set of reals which have a prefix in $Q$.
We may identify each Borel probability measure on  $2^{\omega}$
by its measure representation, \ie a function $\mu: 2^{<\omega}\to [0,1]$ (determining 
its values on the basic open sets) with the property 
$\mu(\sigma)=\mu(\sigma\ast 0)+\mu(\sigma\ast 1)$ for each $\sigma\in 2^{<\omega}$,
which maps the empty string to 1. Given  set of strings $C$, we let $\mu(C)$ denote
the measure of the corresponding open set in the Cantor space, which  equals
$\sum_{\sigma\in C} \mu(\sigma)$ in the particular case that $C$ is prefix-free.

Let us fix a
universal enumeration $(\mu_i)$ of all {\em partial computable measures} 
which we view as the partial computable functions $\mu$ with the property
that (a) they are defined on the empty string and equal 1, and 
(b) for all $\sigma$, if $\mu(\sigma\ast i)\de$ for some $i\in\{0,1\}$ then
$\mu(\rho)\de$ for all strings of length at most $|\sigma|+1$
and $\mu(\sigma)=\mu(\sigma\ast 0)+\mu(\sigma\ast 1)$.
Then clearly $(\mu_i)$ contains all computable measures.
We use the suffix `$[s]$' to denote the state of an object after $s$ steps of computation. 
Given a \pf set $C$ of strings and $i$, we say that $\mu_i(C)[s]\de$ if
$\mu_i(\sigma)[s]\de$ for all $\sigma\in C$.
Without loss of generality, in our universal enumeration $(\mu_i)$ we assume that 
if $\mu_i(\sigma)[s]\de$ then $\mu_i(\tau)[s]\de$ for all strings $\tau$ of length at most $|\sigma|$.
Each $\mu_i$ has a {\em time-complexity} (possibly partial) function
which maps each $n$ to the least stage $s$ such that $\mu_i(2^{\leq n})[s]\de$.
Generally speaking we are interested in {\em continuous measures} \ie measures $\mu$ such
that $\mu(\{X\})=0$ for each real $X$. 

\subsection{Algorithmic randomness with respect to computable measures}\label{GY8YN2nY9W}
Bienvenu and Merkle \cite{BIENVENU2009238} 
contains an informative presentation of
algorithmic randomness with respect to computable measures.
Here we recall the basic concepts and facts on this topic that are directly related to
our arguments.
A \ml $\mu$-test is a uniformly computably enumerable sequence $(U_i)$ of sets of strings
such that $\mu(U_i)\leq 2^{-i}$ for each $i$. We say that $X$ is $\mu$-random for
a computable measure $\mu$ if $X\not\in \cap_i \dbra{U_i}$ for all \ml $\mu$-tests
$(U_i)$. In the case where a computable measure $\mu$ is continuous (\ie it does not have
atoms) the theory of \ml $\mu$-randomness is entirely similar to the theory of \ml randomness
with respect to the uniform measure. For example, by Levin 
\cite{leviniandc/Levin84} we have the following characterization in terms of the
\pf initial segment Kolmogorov complexity $n\mapsto K(X\restr_n)$:
\begin{equation}
\parbox{10cm}{Given a computable measure $\mu$, a real 
$X$ is \ml $\mu$-random if and only if 
$\exists c\ \forall n\ K(X\restr_n)\geq -\log \mu(X\restr_n)-c$.}
\end{equation}
An important concept for the proof of both of the clauses of
Theorem \ref{UejsgfHHK} is the randomness deficiency of a real $X$
with respect to a computable measure $\mu$. There are different definitions
of this notion, but most of them are effectively equivalent (in a way to be made precise
in the following) and are based on the same intuition: 
\begin{itemize}
\item $X$ is $\mu$-random if and only if it has bounded (\ie finite) $\mu$-randomness deficiency;
\item the more  $\mu$-randomness deficiency $X$ has, the less $\mu$-random $X$ is.
\end{itemize}

It will help the uniformity of our treatment to deal with the partial computable measures
and regard totality as a special case. In this respect,
we define the {\em randomness deficiency functions} as a uniform
sequence of partial computable functions $(d_e)$ corresponding to $(\mu_e)$ as follows:
\[
d_e(\sigma)=-\log \mu_e(\sigma)-K(\sigma)
\hspace{0.5cm}\textrm{for each $e,\sigma$.}
\]
Then we can also define the
randomness deficiency functions on reals as the sequence $(\mathbf{d}_e)$
defined as:
\[
\mathbf{d}_e(X)=\sup_n d_e(X\restr_n)
\hspace{0.5cm}\textrm{for each $e,X$}
\]
where the supremum is taken over the $n$ such that 
$d_e(X\restr_n)\de$ (hence, at least $n=0$).
In this way, the $\mu$-randomness deficiency 
of $\sigma$ if the amount that $\sigma$ can be compressed by
the underlying universal machine, compared to its default
code-length $-\log \mu_e(\sigma)$ which is chosen according to $\mu_e$.
Similarly, the $\mu$-randomness deficiency of $X$ is the maximum
amount by which the initial segments of $X$ are compressible.

Alternatively, we could have defined the $\mu$-randomness deficiency
of $X$ as the least $i$ such that $X\not\in \dbra{U_i}$, where $(U_i)$
is a fixed universal \ml $\mu$-test; this definition can also be made 
uniform in the indices of the partial computable measures.
The intuition behind this alternative deficiency notion is that
effectively producing a $\mu$-small open neighborhood of $X$ increases the
randomness deficiency of $X$. This interpretation will be crucial in Section \ref{cTQXDYoKpF},
and we will later observe that it is essentially equivalent to the definition 
of $\mathbf{d}_e(X)$ that we gave in terms of Kolmogorov complexity.

\subsection{Proof of the first clause of Theorem \ref{UejsgfHHK} (easier)}
Recall that for any oracle $A$ we have $A'\geq_T \emptyset''$
if and only if $\Pi^0_2\subseteq\Delta^0_2(A)$, which means that
the answers to any uniformly $\Pi^0_2$ sequence of questions, such as
the totality of partial computable functions, can be approximated  by a function which 
is computable in $A$. Hence, also in view of the domination result of Martin mentioned earlier,
the first clause of  Theorem \ref{UejsgfHHK}  can be stated as follows:
\begin{equation}\label{RGgJuzFZpA}
\parbox{13cm}{Let $f$ be a function with binary values such that for each $e$ we have
$\lim_s f(e,s)=1$, if $\mu_e$ is total and $\lim_s f(e,s)=0$ if $\mu_e$ is partial.
Then there exists a learner $\LL$ which is computable in $f$ and which
$EX$-succeeds on all computable measures.}
\end{equation}
The main idea is to first observe that given a uniformly computable
sequence $(\lambda_i)$ of total measures, we can 
define the randomness deficiency functions $(d_i)$ as in Section \ref{GY8YN2nY9W}
(but with respect to $(\lambda_i)$ instead of $(\mu_i)$) and
these will be total. Hence  we can define
the computable learner 
which monitors the deficiencies along each real $X$ and at each step $n$
predicts the index $i$ which minimizes the cost $d_i(X\restr_n)[n]+i$.
It is easy to check that this learner succeeds 
in all of the measures $(\lambda_i)$.\footnote{Since Definition 
\ref{pRSkquPLEN} refers to the universal indexing $(\mu_i)$,
at this point the reader may be concerned with the
difference in the indexing $(\lambda_i)$. However this is not an issue,
since there is a computable map from the
indices in special list $(\lambda_i)$ to the corresponding items in
the universal list $(\mu_i)$ that we fixed.}
Second, with an oracle which decides totality of partial computable functions {\em in the limit},
we can implement a similar learner for the universal list $(\mu_i)$, 
by eventually identifying and ignoring the partial members of 
$(\mu_i)$ in our calculations of the costs $d_i(X\restr_n)[n]+i$.

Formally, for each $\sigma$ define 
\[
\texttt{cost}(\sigma,e)[s]=
e+\max \{d_e(\sigma\restr_{n})[s]\ |\ n\leq |\sigma|\ \wedge\ d_e(\sigma\restr_{n})[s]\de\}
\]
where the maximum of the empty set is defined by default to be 0.
Then for each $\sigma$
\[
\LL(\sigma) = \min \big\{i\leq |\sigma| \ |\ 
f(i,|\sigma|)=1 \wedge\ 
\texttt{cost}(\sigma,i)[|\sigma|]=\min_{j\leq |\sigma|} \texttt{cost}(\sigma,j)[|\sigma|]\big\}  
\]
\ie $\LL(\sigma)$ is the least $i\leq |\sigma|$ which minimizes 
the cost of $\sigma$.

Clearly $\LL$ is computable in $f$, so it remains to show that for every $X$ which 
is $\mu_j$-random for some total $\mu_j$, the limit $\lim_n \LL(X\restr_n)$ exists
and is a number $i$ such that $\mu_i$ is total and $X$ is $\mu_i$-random.
Consider the least number $e$ which minimizes the expression
$\mathbf{d}_e(X)+e$
amongst the indices of total computable measures.
It remains to show that 
$\lim_n \LL(X\restr_n)=e$.
Note that by our hypothesis about $X$, $\mathbf{d}_e(X)+e$ is a finite number.
Let $s_0$ be a stage such that 
\begin{enumerate}[\hspace{0.5cm}(a)]
\item $f(j,s)=\lim_t f(j,t)$ for all $j\leq \mathbf{d}_e(X)+e$
\item $j+\max_{i\leq s} d_j(X\restr_i)[s]\geq\max_{i\leq s} d_e(X\restr_s)[s]+e$
for all $s\geq s_0$ and $j\leq \mathbf{d}_e(X)+e$ such that $\mu_j$ is total;
\item $j+\max_{i\leq s} d_j(X\restr_i)[s]>\max_{i\leq s} d_e(X\restr_s)[s]+e$
for all $s\geq s_0$ and $j<e$ such that $\mu_j$ is total.
\end{enumerate}

By the choice of $X$ and the definition of $\mathbf{d}_e(X)$
we have $\LL(X\restr_n)\leq \mathbf{d}_e(X)+e$ for all $n$.
If $\LL(X\restr_n)=j$ for some $n>s_0$, then by the choice of $s_0$
the measure $\mu_j$ is total, so by clause (b)  above,
and the minimality in the definition of $\LL$ on $X\restr_n$, we must have $j\leq e$.
Then by clause (c) and the definition of $\LL$ it is not possible that $j< e$, so $j=e$
and this shows that  $\LL(X\restr_n)=e$ for all $n>s_0$, 
which concludes the proof of \eqref{RGgJuzFZpA}.

\subsection{Proof of  Theorem \ref{8m2M82eMhm} about partial learning}\label{r9nX2v56yb}
Let $\ell_i[s]$ be the largest number $\ell$ such that $\mu_i(2^{\leq\ell})[s]\de$.
A stage $s$ is called $i$-expansionary if 
$\ell_i[t]<\ell_i[s]$ for all  $i$-expansionary stages $t<s$.
By the padding lemma let $p$ be a computable function such that for each $i,j$ 
we have $\mu_{p(i,j)}\simeq \mu_{i}$ and $p(i,j)<p(i,j+1)$.

At stage $s$, we define $\LL(\sigma)$ for each $\sigma$ of length $s$ as follows.
For the definition of  $\LL(\sigma)$ find the least $i$ such that $s$ is $i$-expansionary and
$d_i(\sigma)[s]\leq i$. Then let $j$ be the least
such that $p(i,j)$ is larger than any $k$-expansionary stage $t<|\sigma|$ for any $k<i$
such that $d_k(\sigma\restr_k)[t]\leq k$,
and define $\LL(\sigma_t)=p(i,j)$.

Let $X$ be a real. Note that $\LL(X\restr_n)=x$ for infinitely many $n$, then
$x=p(i,j)$ for some $i,j$, which means that $\mu_i=\mu_x$ is total
and there are infinitely many $x$-expansionary stages as well as infinitely many
$i$-expansionary stages.
This implies that there are at most $x$ many 
$y$-expansionary stages $t$ for any
$y<x$ with $d_y(\sigma\restr_y)[t]\leq y$. Moreover for each $z>x$
there are at most finitely may $n$ such that $\LL(X\restr_n)=z$. Indeed, for each $z$
if $n_0$ is an $i$-expansionary stage then $\LL(X\restr_n)\neq z$ for all $n>n_0$.
Moreover if  $\LL(X\restr_n)=x$ for infinitely many $n$, then 
$d_x(X)=d_i(X)\leq i$ and $\mu_i$ is total, so $X$ is $\mu_i$-random.
We have shown that for each $X$ there exists at most one $x$ such that
$\LL(X\restr_n)=x$ for infinitely many $n$, and in this case $\mu_x$ is total
and $X$ is $\mu_x$-random.

It remains to show that if $X$ is $\mu$-random for some computable $\mu$, then
there exists some $x$ such that $\LL(X\restr_n)=x$ for infinitely many $n$.
If $X$ is $\mu_i$-random for some $i$ such that $\mu_i$ is total,
let $i$ be the least such number 
with the additional property that $\mathbf{d}_i(X)\leq i$ (which exists by the padding lemma).
Also let $j$ be the least number such that 
$p(i,j)$ is larger than any stage $t$ which is 
$k$-expansionary for any $k<i$
with $d_k(\sigma\restr_k)[t]\leq k$. Then the construction will define
$\LL(X\restr_n)=p(i,j)$ for each $i$-expansionary stage $n$ after
the last $k$-expansionary stage $t$ for any $k<i$
with $d_k(\sigma\restr_k)[t]\leq k$.
We have shown that $\LL$ partially succeeds on every $\mu$-random $X$ for any
computable measure $\mu$.

\section{Proof of the second clause of Theorem \ref{UejsgfHHK} (harder)}\label{cTQXDYoKpF}
In order to make $\LL$ compute a function $f$ which dominates every computable function, 
the idea is to use the convergence times of the current guesses (\eg for the strings of length $s$)
of $\LL$ in order to produce the large number $f(s)$. The immediate problem is that some of
the current guesses may point to partial measures $\mu_i$, so the search of some convergence
times may be infinitely long. Although we cannot decide at stage $s$ which of these guesses
are such, we know that they are erroneous guesses, and they cannot be maintained
with positive probability, with respect to any computable probability measure $\mu$.
Hence for each such guess $\mu_i$ (on a string $\sigma$ of length $s$) we can wait for either
the convergence of $\mu_i\restr_s$ or the change of the $\LL$-prediction in
a sufficiently `large' set of extensions of $\sigma$.
In order to make this idea work, we would need to argue that 
\begin{equation}\label{W9SmgfZoCl}
\parbox{13cm}{for each computable
function $h$, the failure of $f\leq_T \LL$ to dominate $h$ means that for some computable
measure $\lambda_h$ the learner
$\LL$ fails to give correct predictions in the limit for a 
set of reals of positive $\lambda_h$-measure.}
\end{equation}
In order to make these failures concrete, in 
Section \ref{VUDNFuv7XR} we show that without loss of generality
we may assume that $\LL$ does not only predict a measure $\mu_i$ along each real $X$,
but also an upper bound on the $\mu_i$-randomness deficiency of $X$.
Then the crucial lemma which allows the above argument for the domination of $h$ from 
$f\leq_T$ to succeed is the following
fact, which we prove in Section \ref{9gEizhkpFI}: 
\begin{equation}\label{t5sq9gm2h}
\parbox{14.8cm}{for every computable
function $h$ there exists a computable measure $\lambda_h$ and a $\lambda_h$-large class
of reals $\VV_h$ such that for any $X\in\VV_h$ and $\mu_i$ that may be the suggested hypothesis
of $\LL$ along $X$, either the time-complexity of $\mu_i$ dominates $h$ or $X$ has
large $\mu_i$-deficiency (above the guess of $\LL$).}
\end{equation}
We will also need to make sure that the measure  $\lambda_h$ that we design from $h$
is relatively identical to the uniform measure for all strings  in long intervals of lengths,
compared to the growth of $h$.
This feature will allow us to know how long to wait for the convergence of
$\mu_i\restr_s$ for the guess $\mu_i$ of $\LL$ on some $\sigma$ at stage $s$,
\ie on `how many' reals extending $\sigma$ does $\LL$ have to change its guess
before we give-up on the convergence of $\mu_i\restr_s$. Indeed, although
this size of reals is supposed to be with respect to $\lambda_h$, the definition of $f$ should
not depend directly on $\lambda_h$ since the totality of $h$ and hence of $\lambda_h$ cannot
be determined by $\LL$ at each stage $s$.

Then the crucial positive $\lambda_h$-measure set in the main argument 
\eqref{W9SmgfZoCl} will be a subset of $\VV_h$, namely all the reals in $\VV_h$ 
except for the open sets $\MM_h(s)$ of 
reals for which we did not weight long enough at the various stages $s$
of the definition of $f$, before we give up waiting for the convergence of the
current guesses. The tricky part of the construction of $f$, which we present and verify
in Section \ref{3hp9XGohn}, is to ensure that a positive $\lambda_h$-measure remains in
$\VV_h$, despite the fact that $\lambda$ is not available in the construction in order to 
directly measure the sets $\MM_h(s)$ that need to be removed from $\VV_h$.
 
\subsection{Randomness deficiency and learning}\label{VUDNFuv7XR}
The proof of the second clause of Theorem \ref{UejsgfHHK}
will be based on the fact that 
\begin{equation}\label{rlOqOBSNx}
\parbox{14cm}{if a learner 
learns a computable measure $\mu$ along a real $X$, such that $X$ is
$\mu$-random, then it can also learn an upper bound on the randomness deficiency
of $X$ with respect to $\mu$.}
\end{equation}

We call this notion {\em strong EX-learning along} $X$ and by the padding lemma (the
fact that one can effectively produce arbitrarily large indices of any given computable measure)
we can formulate it as follows.

\begin{defi}[Strong EX-learning]\label{rU3irtSqYN}
Given a class of computable measures $\CC$, a learner $F$ and a real $X$,
we say that the learner strongly EX-succeeds on $X$ if
$\lim_n F(X\restr_n)$ exists and equals an index $i$ of some
$\nu\in\CC$ such that the $\nu$-randomness deficiency of $X$ is bounded above by $i$.
Given $\mu\in\CC$ we say that $F$ strongly $EX$-succeeds on $\mu$ if
it strongly $EX$-succeeds on every
$\mu$-random real $X$.
\end{defi}

Then we can write \eqref{rlOqOBSNx} as follows.
\begin{lem}\label{NhUZ5ZOtU5}
Given a class of computable measures $\CC$ and a learner $F$,
there exists a learner $F^{\ast}$ 
which strongly $EX$-succeeds on every real $X$, on which the given
learner $EX$-succeeds.
\end{lem}
\begin{proof}
Let $g$ be a computable function such that for each $i,t$
the value $g(i,t)$ is an index of $\mu_i$ and $g(i,t)>t$.
Recall the definition of $d_i(\sigma)[s]$ 
from Section \ref{GY8YN2nY9W}.
Define $F^{\ast}(\sigma)=g(F(\sigma), d(F(\sigma),\sigma)[|\sigma|])$.
Given $X$, suppose that $\lim_n F(X\restr_n)=i$  and $X$ is $\mu_i$-random.
Then $\lim_n d_i(X\restr_n)[n]<\infty$ so $\lim_n F^{\ast}(X\restr_n)$ exists and
is an index $j$ such that $\mu_j=\mu_i$. Moreover by the definition of $g$ we have
that $j> d_j(X)$ which show that $F^{\ast}$ strongly $EX$-succeeds on $X$.
\end{proof}

In the crucial lemma \eqref{t5sq9gm2h} that we prove in the next section,
we will need to increase the randomness deficiency of some reals, which we do
through {\em tests}. Recall we have fixed a certain effective sequence $(\mu_e)$ 
of partial computable measures which includes all computable measures. 
For each $e$ we define a {\em clopen $\mu_e$-test} to be a partial computable
function $i\to D_i$ from integers to finite sets of strings (described explicitly)
with the properties that for each $i$, if $D_i\de$ then $D_j\de$, $\mu_e(D_j)\de$ 
and $\mu_e(D_j)\leq 2^{-j}$ for each $j\leq i$. Note that we incorporate partiality
in the definition of  clopen tests. In this way, there exists an effective universal list
of all clopen $\mu_e$-tests for each $e$, so we may refer to
{\em an index} of a clopen test, in relation the universal list. 
The following fact is folklore.
\begin{equation}\label{EvZSYvwTm6}
\parbox{14cm}{Uniformly in a randomness 
deficiency bound $k$, an index of a partial computable measure $\mu_i$ and
the index of a $\mu_i$-clopen test, we can compute the index of a member of the test
which, if defined, the strings in it have deficiency exceeding $k$.}
\end{equation}
In combination with the recursion theorem, \eqref{EvZSYvwTm6} says that
when we construct a $\mu$-clopen test for some partial computable measure
$\mu$, we can calculate a lower bound on the indices of the members test 
that guarantees sufficiently high 
(larger than the prescribed value $k$) 
$\mu$-randomness deficiency of all of the strings in them.

We state and prove the version of \eqref{EvZSYvwTm6}
which will be used in the argument of Section \ref{9gEizhkpFI}.
A {\em uniform sequence of clopen tests} $(\GG^i)$ is a uniformly 
computable sequence of  clopen tests $\GG^i=(G^i_j)_{j\in\Nat}$
such that for each $i$, $\GG^i$ is a clopen $\mu_i$-test.
As before, since we incorporated partiality in the definition of clopen tests,
each uniform sequence of clopen tests are indexed in a fixed universal enumeration of all
uniform sequences of clopen tests.
The main argument in Section \ref{9gEizhkpFI} will be the construction of a
$\mu_i$-test for each $i$, hence  uniform sequence of clopen tests,  which control
the randomness deficiencies of a set of reals through the following fact.

\begin{lem}[Uniformity of randomness deficiency]\label{LAaax2mHgG}
There is a computable function $q(t,e,k)$ which takes any index $t$ of a
uniform sequence of clopen tests $\GG^i=(G^i_j)$ and any number $k$,
and for each $e$ we have
$d_e(\sigma)>k$ for all $\sigma\in G^e_{q(t,e,k)}$, provided
that $\mu_e(G^e_{q(t,e,k)})$ is defined.
\end{lem}
\begin{proof}
We construct a \pf machine $M$ and a function $q$, and by the recursion theorem we may use
the index $x$ of $M$ in the definition of $M,q$.\footnote{Formally, we run the construction with a
free parameter $y$ which is treated as if it is the index of the machine that we are constructing
(although it may not be) thus producing a computable function $h$ which gives
the uniform sequence of machines $M_{h(y)}$ (given a universal enumeration $M_i$ of all
\pf machines). In the argument here, the weight of each $M_{h(y)}$ is explicitly forced to be at most
1, thereby stoping any further computation if this bound is reached. This ensures that
for each $y$ the process defines a prefix-free 
machine, even if $y$ is not an index of the machine constructed.
Then by the recursion theorem we choose $x$ such that $M_{h(x)}=M_x$
and this is the desired index for which the construction correctly assumes that
the input number is an index of the machine being constructed.} 
We define $q(t,e,k)=t+e+k+x+5$ and the machine $M$ as follows.
For each index $t$ we let $\GG^i=(G^i_j)$ be the uniform sequence of clopen tests
with index $t$ and
for each $e,k$ we wait until $\mu_e(G^e_{q(t,e,k)})\de$; if and when this happens
we compress all strings $\sigma\in G^e_{q(t,e,k)}$ using $M$,
so that $K_M(\sigma)= -\log\mu_e(\sigma)- k-x$ (and stop these compressions if and when
the compression cost exceeds 1). 
Since $\mu_e(G^e_{q(t,e,k)})<2^{-t-e-x-k-5}$
the cost of this operation is $2^{-t-e-k-x-5+k+x}=2^{-t-e-5}$ and the total cost is
$\sum_{t,e} 2^{-t-e-5} <1$
so the compression of $M$ as we described it will not be stopped due to an excess of the
compression cost. Then by the definition of $M$, for each
$t,e$ such that $\mu_e(G^e_{q(t,e,k)})\de$ and each $\sigma\in G^e_{q(t,e,k)}$,
\[
K(\sigma)\leq K_M(\sigma)+x=-\log\mu_e(\sigma) +x - k -x=-\log\mu_e(\sigma) - k 
\]
so $d_e(\sigma)\geq k$ which concludes the proof.
\end{proof}
Informally, the computable function $p$ of Lemma \ref{LAaax2mHgG} will tell
the construction of Section \ref{9gEizhkpFI}, how long it needs to build the 
$\mu_e$-test $\GG^e$ for each $e$, in order to achieve the required 
$\mu_e$-randomness deficiency. 


\subsection{The domination lemma}\label{9gEizhkpFI}
For the verification of the dominating function from the learner, in Section \ref{3hp9XGohn},
we need a lemma which says, roughly speaking, that for each computable function $h$ there exists
a computable measure whose time-complexity is higher than $h$ and which resembles the uniform
measure except for a very sparse set of strings.\footnote{A similar tool was used in
\cite{AdlemanB91}, namely a version of the fact from \cite{Blum:1967:MTC} that
given any computable function $h$ there exists a computable function
such that any implementation of it converges in time exceeding $h$.}
Before we state the lemma, we need the following definitions.

\begin{defi}[Relative equality of measures in intervals]
Given two measures $\mu,\lambda$ we say that $\mu$ is relatively equal (or identical) 
to $\lambda$
in an interval $[a,b]$ if for all strings $\sigma$ of length in $[a,b)$ and each $j\in\{0,1\}$ we have
$\mu(\sigma\ast j)/\mu(\sigma)=\lambda(\sigma\ast j)/\lambda(\sigma)$.
\end{defi}

Note that if 
$\mu$ is relatively identical 
to $\lambda$ in $[a,b]$ then 
for each $s,t\in [a,b]$ with $s<t$, each $\sigma\in 2^s$ 
and each \pf set $D$ of strings in $2^{\leq t}$ extending $\sigma$,
we have $\mu_{\sigma}(D)= \lambda_{\sigma}(D)$, where
$\mu_{\sigma}(D):=\mu(D)/\mu(\sigma)$ is the conditional $\mu$-measure
with respect to $\sigma$, and similarly with $\lambda_{\sigma}$.

\begin{defi}[Sparse measures]
Given an increasing function $h$, a sequence $(n_i)$ is $h$-sparse if $h(n_i)<n_{i+1}$
for each $i$. A measure $\lambda$ is $h$-sparse if there exists an
$h$-sparse sequence  $(n_i)$ such that for each $i$, $\lambda$ is relatively identical  to the
uniform measure in $(n_i, n_{i+1}]$.
\end{defi}

The following crucial lemma is exactly what we need for the argument of Section \ref{3hp9XGohn}.
\begin{lem}[Time-complexity of computable measures]\label{ADVYKUBE4x}
Given computable functions $h,g,p$ 
there exists a $p$-sparse computable measure $\lambda$ and
a $\Pi^0_1$ class of reals $\CC$ 
such that $\lambda(\CC)=1$ and
for every $X\in\CC$ and every index $i$ of a computable measure 
$\mu_i$ such that the $\mu_i$-deficiency of $X$ 
is $\leq g(i)$, the time-complexity of $\mu_i$ dominates $h$.
\end{lem}

We can easily derive Lemma \ref{ADVYKUBE4x} from the following technical lemma,
whose proof we give in Section \ref{suKJfH4yOo}.

\begin{lem}\label{j1A3ktur8B}
Given any computable functions $h,g,p$ 
there exists an $p$-sparse computable measure $\lambda$,
a $\Pi^0_1$ class of reals $\CC$ 
such that $\lambda(\CC)=1$ and
a uniform sequence of clopen tests $\GG^i=(G_j^i)$ 
such that, if $\ell_i$ is the length of $\GG^i$, we have:
\begin{enumerate}[\hspace{0.5cm}(a)]
\item $\mu_i(G_j^i)\leq 2^{-j}$ and $\CC\subseteq \dbra{G_j^i}$ for all $j< \ell_i$;
\item if the time-complexity of $\mu_i\restr_n$ does not dominate $h$
then the length of $\GG^i$ is $g(i)$.
\end{enumerate}
In particular, $\GG^i$ is a $\mu_i$-test for each $i$ such that $\mu_i$ is total,
and $\mu_i(V)\leq 2^{-g(i)}$ if in addition clause (b)  holds.
\end{lem}

Note that the hypothesis in clause (b) above implies that $\mu_i$ is total.

We prove Lemma \ref{ADVYKUBE4x} so fix $h_0,g_0,p_0$ in place of
$h,g,p$  in the statement of the lemma. 
In order to obtain the desired $\CC,\lambda$ corresponding to 
the given $h_0,g_0,p_0$,
we are going to apply
Lemma \ref{j1A3ktur8B} on $h=h_0,g=g_1,p=p_0$ where 
$g_1$ is an appropriate function related to $g_0$, which we are going to obtain as 
a fixed-point.
If we fix $h=h_0,p=p_0$ and regard (the index of) $g$ as a variable in
Lemma \ref{j1A3ktur8B}, we get a
total effective index-map $j\mapsto v(j)$
from any index of a function $g$, to an index 
$v(i)$ of a uniform sequence of clopen tests $\GG^i$, 
such that if the function $g$ with index $j$ is total, the stated properties
hold.
Now recall the function $q$ of Lemma \ref{LAaax2mHgG}. 
By the recursion theorem and the s-m-n theorem we can choose $j_0$ such that
the function with index $j_0$ is the same as the function 
$e\mapsto q(v(j_0),e,g_0(e))$. In particular, this function it total 
(because $e\mapsto q(v(j),e,g_0(e))$ is total for every $j$) so we can define
$g_1$ to be the function with index $j_0$. Then the properties of $q$
and the uniform sequence of clopen tests indexed by $v(j_0)$ according to 
Lemma \ref{j1A3ktur8B}, shows that 
Lemma \ref{j1A3ktur8B} applied to $h=h_0, g=g_1, p=p_0$ produces $\lambda, \CC$
satisfying the properties of 
Lemma \ref{ADVYKUBE4x} for $h=h_0, g=g_0, p=p_0$.

\subsection{Proof of Lemma \ref{j1A3ktur8B}}\label{suKJfH4yOo}
Given $h,g,p$ we produce a sequence $(C_s)$ of finite sets of strings
such that the strings in $C_s$ have length $s$ and every string in $C_{s+1}$ has
a prefix in $C_s$. 
Then we define $\CC$ to be the set of reals which 
have a prefix in each of the sets $C_s$, \ie 
$\CC=\cap_s \dbra{C_s}$
which is clearly a $\Pi^0_1$ set, and ensure that $\CC$ is $p$-sparse.
Without loss of generality we assume that $p$ is increasing.

We also define a uniform sequence of clopen tests $(G^i_j)$ 
whose members will be defined to be $C_s$ for certain $s$.
In particular, if stage $s$ {\em acts for} $i$ then the next member of  $\GG^i$ will be defined
equal to $C_s$.
In this sense, certain members of $(C_s)$ become members of the 
tests $\GG^i$ in a one-to-one fashion (in the sense that each $C_s$ is assigned to at most
one $\GG^i$). In order to make $\lambda$ a $p$-sparse measure, if stage $s$ acts for
some $i$ (hence, it is an active stage) the next $p(s)$ many stages of the construction
are {\em suspended}, in the sense that no action is allowed on those stages. 
Define $\lambda$ on the empty string equal to 1.

 At each stage $s$ some $i<s$ may require attention, in which case $s$ will be
 an $i$-stage for the least such $i$, and a new member of $\GG^i$ will 
 be defined equal to $C_s$.
We say that $i<s+1$ {\em requires attention} at stage $s+1$ if
it has received attention less than $g(i)$ many times and
\begin{itemize}
\item $\mu_i\restr_s[h(s+1)]\de$;
\item for all $j<i$ and all $n\leq (s+1, f(s+1)]$ we have $\mu_j\restr_n[h(n)]\un$
or $j$ has acted $g(j)$ times.
\end{itemize}
Intuitively, for $i$ to require attention at stage $s+1$ we require first that
it has received attention less than $g(i)$ many times, second that 
we can determine the values of $\mu_i$ for strings up to length $s+1$ (which are the
strings from which we are about to choose the members of 
$C_{s+1}$) and third, that
in none of the future stages that are about to become suspended as a result to
$\mu_i$ receiving attention at this stage, some $\mu_j$ with $j<i$ might require attention.

{\em Construction of $\lambda$ and $(C_s)$.} At stage $s+1$, if this is a suspended stage
or no $i\leq s$ requires attention, 
\begin{itemize}
\item let $C_{s+1}=\{\sigma\ast 0, \sigma\ast 1\ |\ \sigma\in C_s\}$;
\item for each $\sigma\in C_s$ and $j\in\{0,1\}$ let $\lambda(\sigma\ast j)=\lambda(\sigma)/2$;
\end{itemize}
 and 
go to  stage $s+2$.
Otherwise,  consider the least $i\leq s$ which requires attention
and for each string $\sigma$ of length $s+1$ do the following:  
\begin{itemize}
\item define $j_{\sigma}\in\{0,1\}$ to be
the least such that $\mu_i(\sigma\ast j_{\sigma})\leq \mu_i(\sigma\ast (1-j_{\sigma}))$
\item for each $\sigma\in C_{s}$ define $\lambda(\sigma\ast  j_{\sigma})=0$;
and $\lambda(\sigma\ast (1- j_{\sigma}))=\lambda(\sigma)$;
\item for each $\sigma\in 2^s-C_{\sigma}$ define 
$\lambda(\sigma\ast  0)=\lambda(\sigma\ast  1)=\lambda(\sigma)/2$;
\item define $C_{s+1}=\{\sigma\ast (1-j_{\sigma})\ |\ \sigma\in C_{s}\}$.
\end{itemize}
Declare stage $s+1$ active 
(acting on $i$) and the next $p(s+1)$ stages suspended.
Also let $j$ be the least such that $G^i_j\un$ and define
$G^i_j=C_{s+1}$.

\begin{lem}
The function $\lambda$ is a continuous computable $p$-sparse measure  and
$\lambda(\CC)=1$.
\end{lem}
\begin{proof}
Inductively it follows that $\lambda$ is a measure representation and since the construction
is effective, it is also computable. Also if $(n_i)$ are the active stages, the
construction shows that $\lambda$ is relatively identical to the uniform measure 
in each of the intervals 
$[n_i,n_{i+1}]$, since it only differs from the uniform 
measure on the strings of lengths $n_i, i\in\Nat$.
Moreover by construction $(n_i)$ is $p$-sparse, which shows that $\lambda$ is also
$p$-sparse. Clearly there are infinitely many stages which are either suspended or
no $i$ requires attention, so $\lambda$ is continuous. 
Finally at each step that $\dbra{C_{s+1}}$ is smaller than $\dbra{C_s}$, 
all the $\lambda$-measure of  $\dbra{C_s}$ is transferred to $\dbra{C_{s+1}}$. Hence
$\lambda(C_s)=1$ for all $s$, which shows that $\lambda(\CC)=1$.
\end{proof}
Note that the sequence $G^i_j$ is uniformly partial computable, in the sense that
the function $(i,j)\mapsto G^i_j$ is partial computable.
\begin{lem}
For each $i,j$, if $G^i_{j+1}\de$ then $G^i_{j}, \mu_i(G^i_{j+1}), \mu_i(G^i_{j})$
are defined, $\CC\subseteq\dbra{G^i_{j+1}}$  and
$\mu_i(G^i_{j+1})\leq  \mu_i(G^i_{j})/2$. 
In particular, for each $i$, $\GG^i$ is a $\mu_i$-test.
\end{lem}
\begin{proof}
By the construction, if $G^i_{j+1}\de$ then  
necessarily $G^i_{j}$ is defined. Moreover, in this case,
if $G^i_{j+1}$ was defined at some stage $s$, we have
$G^i_{j+1}=C_s$ and stage $s$ acted on $i$, which means that $\mu_i(C_s)[s]\de$ so
$\mu_i(G^i_{j+1})[s]\de$. 
By the definition of $\CC$ as the intersection of all $C_t, t\in\Nat$, we also have
$\CC\subseteq\dbra{G^i_{j+1}}$.
Also since the strings in $C_s$ are longer than the strings in
any $C_t, t<s$, by the same argument we have
$\mu_i(G^i_{j})[s]\de$.
Now let $t<s$ such that $G^i_{j}=C_{t}$ and note that by construction we have
$\mu_i(C_{s-1})\leq \mu_i(C_{s})/2$. Since $t\leq s-1$ we have 
$\dbra{C_t}\subseteq\dbra{C_{s-1}}$ and
\[
\mu_i(G^i_{j})=\mu_i(C_{t})\leq \mu_i(C_{s-1})\leq \mu_i(C_{s})/2=\mu_i(G^i_{j})/2
\]
which concludes the proof of the lemma.
\end{proof}
It remains to show clause (b) of Lemma \ref{j1A3ktur8B}.
If $\mu_i$ is not total then clearly the time-complexity of $\mu_i\restr_n$ will be equal to 
infinity co-finitely often, hence dominating $h$. hence it suffices to show the following
fact by a finite injury argument.
\begin{lem}
If $\mu_i$ is total and its time-complexity does not dominate $h$, then
the length of $\GG^i$ is $g(i)$.
\end{lem}
\begin{proof}
By the definition of `requiring attention', each $j$ will receive attention at most finitely many times.
We argue that each $j$ will also require attention at most $g(i)$ many times. Indeed, otherwise
there would be a least $j$ which requires attention infinitely many times, and by the construction
it will receive attention infinitely many times, which is a contradiction.

Let $s_0$ be a stage after which no $\mu_j, j\leq i$ requires attention, and let
$s_1=f(s_0)$.
It suffices to show that the length of  $\GG^i[s_1]$ is $g(i)$. Note that each time that $i$ receives
attention, the length of (the current state of) $\GG^i$ increases by 1, so the length of 
$\GG^i$ equals the number of times that $i$ receives attention.
Hence the length of $\GG^i$ is at most $g(i)$. 
For a contradiction, assume that 
the length of  $\GG^i$ is less than $g(i)$.
By the hypothesis about $\mu_i$, there will be a least stage $s+1>s_1$ such that
$\mu_i\restr_{s+1}[h(s+1)]\de$. 
Since the length of $\GG^i[s]$ is less than $g(i)$, by the choice of 
$s_0, s_1$, no $j<i$ requires attention at stage $s+1$.

We also argue that $s+1$ is not a suspended stage. Indeed, $s+1$ 
cannot be suspended by some $\mu_j, j<i$ because all of these indices last acted before 
stage $s_0$, which means that they cannot suspend any stage after $p(s_0)=s_1$. 
If $s+1$ was suspended by some $j>i$ then this must have happened at some stage
$t<s+1$ such that $s+1\in (t, p(t)]$.
But according to the construction, $j$ cannot require attention at stage $t$
because $i<j$ and $\mu_i\restr_{s+1}[h(s+1)]\de$ while $s+1\leq p(t)$.
We may conclude that stage $s+1$ is not suspended.

In order to show that $i$ requires attention at stage $s+1$ it remains to show that
for each $j<i$ either $j$ has acted $g(j)$ many times or 
for all $n\leq (s+1, f(s+1)]$ we have $\mu_j\restr_n[h(n)]\un$. Indeed, if 
$j$ had acted less than $g(j)$ many times and the latter condition did not hold, $j$ would
require attention at stage $s+1$, which contradicts our choice of $s_0$. Therefore,
according to the conditions for index $i$ to require attention, $i$ will require attention 
at stage $s+1$ and since no $j<i$ requires attention, it will receive attention at stage $s+1$.
This again contradicts the choice of $s_0$, and concludes the proof that 
at stage $s_1$ the length of $\GG^i$ will be $g(i)$.
\end{proof}

This concludes the proof of Lemma \ref{j1A3ktur8B}.

\subsection{The dominating function of the learner}\label{3hp9XGohn}
Given a learner $\LL$ which EX-succeeds on all computable measures,
we construct a function $f\leq\LL$ which dominates all computable functions.
Note that
\begin{equation}\label{ZHHgiWeRpT}
\parbox{14.7cm}{for each $\sigma$, letting $e:=\LL(\sigma)$, 
there exists some $n_{\sigma}>2|\sigma|$ such that either 
$\mu_e\restr_{|\sigma|}[n_{\sigma}]\de$ or
the proportion of extensions $\tau$ of $\sigma$ of length $n_{\sigma}$ such that
$\LL(\rho)=\LL(\sigma)$ for all $\rho\in [\sigma,\tau]$ is $<2^{-|\sigma|-5}$.}
\end{equation}
This is because otherwise, $\LL$ would give a partial measure (hence a wrong prediction) 
for a positive measure (with respect to the uniform measure) class of reals.

For each $t$ define $f(t)=\max_{\sigma\in 2^t} F(\sigma)$ where
\[
F(\sigma)=\max\Big\{n_{\sigma}, n_{\tau} \ \big|\ n\in (|\sigma|, n_{\sigma}]\ 
\wedge\ \tau \in\extp{\sigma,n}\Big\}
\]
and $\extp{\sigma,n}$ denotes the strings of length $n$ which extend $\sigma$.

It remains to show that $f$ dominates all computable functions, so fix a computable function 
$h$. Without loss of generality we assume that $h$ is increasing. 
Apply Lemma \ref{ADVYKUBE4x} in order to obtain $\lambda_h$
and $\CC_{h}$ with the stated properties. Since the remaining of the argument
refers to $h$, we drop the subscripts in $\lambda_h$
and $\CC_{h}$ and denote them by
$\lambda$ and $\CC$ respectively.
We also get an increasing sequence $(m_i)$
of active $h$-stages such that $h(m_i)<m_{i+1}$ for each $i$, and 
inside the intervals $[m_i,m_{i+1}]$ the measure $\lambda$ is 
relatively identical to
the uniform measure.
We are going to define a sequence $D_t$ of sets of strings  such that
$\lambda(D_t)<2^{-t-5}$ for each $t$. Then we define
$\CC^{\ast}=\CC-\cup_t \dbra{D_t}$ and show that if $f$ does not dominate
$h$ then $\LL$ fails for $\lambda$ on the set $\CC^{\ast}$. This is a contradiction since
$\lambda(\CC^{\ast})\geq 1/2$.

\begin{defi}[Definition of $D_t$]\label{DSuVREqmbv}
We define $D_t$ by following the definition of $f(t)$.
For each $\sigma\in 2^t$ let $e_{\sigma}=\LL(\sigma)$ and
enumerate extensions of $\sigma$ into
$D_t$ according to the first applicable clause:
\begin{enumerate}[\hspace{0.5cm}(a)]
\item if $f(t)>h(t)$ or $\mu_e\restr_{t}[n_{\sigma}]\de$
then do not enumerate any extension of $\sigma$ into $D_t$;
\item otherwise, if 
there is no $h$-active stage in $(t,n_{\sigma}]$
then we enumerate into $D_t$ all extensions $\tau$ of $\sigma$ of length
$n_{\sigma}$ such that $\LL(\rho)=\LL(\sigma)$ for all $\rho\in [\sigma,\tau]$;
\item otherwise, 
let $n^{\ast}_{\sigma}$ be the least $h$-active stage in $[|\sigma|,n_{\sigma}]$ and 
for each string in $\tau\in \extp{\sigma,n^{\ast}_{\sigma}}$, if 
$\mu_{e_{\tau}}\restr_{n^{\ast}_{\sigma}}[n_{\tau}]\un$ where $e_{\tau}:=\LL(\tau)$,
 enumerate into 
$D_t$ all extensions $\rho$ of $\tau$ of length
$n_{\tau}$ such that $\LL(\eta)=\LL(\tau)$ for all $\eta\in [\tau,\rho]$.
\end{enumerate}
\end{defi}

\begin{lem}[Bounds on the $\lambda$-measure of $D_t$]\label{9wB9Cd69sQ}
For each $t$ we have $\lambda(D_t)< 2^{-t-5}$.
\end{lem}
\begin{proof}
Given $t$, it suffices to show that for each $\sigma$ of length $t$, 
we enumerate into $D_t$ a set of extensions of $\sigma$ of
$\lambda$-measure at most $2^{-t-5}\cdot \lambda(\sigma)$.
Then putting together these bounds for all $\sigma$ of length $t$, we get that
$\lambda(D_t)< 2^{-t-5}$.

So fix $\sigma$.
In case (a) we do not enumerate any extensions of $\sigma$ into $D_t$, 
so there is nothing to prove.
In case (b) we have $\mu_e\restr_{t}[n_{\sigma}]\un$, so by the definition of
$n_{\sigma}$ in \eqref{ZHHgiWeRpT} the proportion of extensions of $\sigma$ of
length $n_{\sigma}$ that are enumerated into $D_t$ is less than $2^{-t-5}$.
But what we really need here is to show that their $\lambda$-measure relative to $\sigma$ is
bounded in this way. In case (b) we also have that
there is no $h$-active stage in 
$[t,n_{\sigma}]$ so $\lambda$ is relatively identical to the uniform measure
in $[t,n_{\sigma}]$. Hence, since the uniform measure of the strings enumerated in $D_t$
relative to $\sigma$ is less than $2^{-t-5}$ the same is true of their $\lambda$-measure.
In other words, we get the bound $2^{-|\sigma|-5}\cdot \lambda(\sigma)$ for these strings
as desired.

We can get the same bound in case (c) as follows. Since $n^{\ast}_{\sigma}$ 
is an $h$-active stage
and $h$-active stages are $h$-sparse, it follows that there is no $h$-active stage in 
$(n^{\ast}_{\sigma}, h(n^{\ast}_{\sigma})]$.
If $n_{\tau}\geq h(n^{\ast}_{\sigma})$ 
for some $\tau$ of length $n^{\ast}_{\sigma}$ extending $\sigma$,
then by the definition of $f$ and the monotonicity of $h$
we would have $f(t)>h(t)$, contrary to our assumption that clause (a) does not apply. 
Hence for each $\tau\in \extp{\sigma,n^{\ast}_{\sigma}}$ we have $n_{\tau}< h(n^{\ast}_{\sigma})$ 
so that there is no $h$-active stage in 
$(n^{\ast}_{\sigma}, n_{\tau}]$.
Now we can use the same argument that we used in case (b) above, but with respect to 
each  $\tau\in \extp{\sigma,n^{\ast}_{\sigma}}$.
If $\mu_{e_{\tau}}\restr_{n^{\ast}_{\sigma}}[n_{\tau}]\un$ where $e_{\tau}:=\LL(\tau)$
then by the definition of $n_{\tau}$
we have the that proportion of the extensions of $\tau$ of length $n_{\tau}$
which are enumerated 
into $D_t$ is less than $2^{-t-5}$,
and since  there is no $h$-active stage in 
$(n^{\ast}_{\sigma}, n_{\tau}]$ the $\lambda$-measure of these strings is less than
$2^{-t-5}\cdot \lambda(\tau)$. 
On the other hand if $\mu_{e_{\tau}}\restr_{n^{\ast}_{\sigma}}[n_{\tau}]\de$ no
extensions of $\tau$ are enumerated into $D_t$ so this bound holds trivially.
Putting together these bounds for all $\tau$ of length $n^{\ast}_{\sigma}$ extending $\sigma$, 
we get the bound 
$2^{-t-5}\cdot \lambda(\sigma)$ on the $\lambda$-measure of the strings in $D_t$
extending $\sigma$, as required.
\end{proof}

\begin{lem}[Domination or leaning failure]\label{pZ2ozweghy}
If $f$ does not dominate $h$ then $\LL$ fails for $\lambda$ on each $X\in\CC^{\ast}$.
\end{lem}
\begin{proof}
Assume that $f$ does not dominate $h$ and let
$X\in\CC^{\ast}$.
For a contradiction assume that $\LL$ does not fail for $\lambda$ on $X$.
As a consequence $\lim_n \LL(X\restr_n)$ exists and is an index $e$
of a computable measure $\mu_e$ such that $X$ is $\mu_e$-random. 
Note that by the discussion of Section \ref{VUDNFuv7XR}
we can assume that $\LL$ strongly EX-succeeds for $\lambda$ on $X$. Hence 
we can pick some $n_0$ such that for all $n\geq n_0$:
\begin{enumerate}[\hspace{0.5cm}(i)]
\item $\LL(X\restr_n)=e$  and $d_e(X\restr_n)\leq e$;
\item if $\mathbf{d}_e(X)\leq e$ then $\mu_e\restr_n[h(n)]\un$;
\end{enumerate}
where the existence of $n_0$ such that the second clause
is met for all $n\geq n_0$ follows by Lemma \ref{ADVYKUBE4x}
applied to $h$, $g(m):=m+2$, $p(m):=h(m)+1$ and $i:=e$.

Now pick some $t>n_0$ such that $f(t)\leq h(t)$, which exists by our hypothesis.
We trace the definition of $f(t)$ in order to derive our contradiction.
Let $\sigma=X\restr_t$ and recall the definition of $n_{\sigma}$ in \eqref{ZHHgiWeRpT}
(where $e$ has the same meaning as here). 
If $\mu_e\restr_{t}[n_{\sigma}]\de$, 
then $\mu_e\restr_{t}[f(t)]\de$ since $f(t)>n_{\sigma}$ and
$\mu_e\restr_{t}[h(t)]\de$ by the choice of $t$.
Then by clause (ii) in the choice of 
$n_0$ we have $\mathbf{d}_e(X)> e$ which contradicts clause (i) of the definition of $n_0$,
since by definition $d_e(X\restr_n)\leq \mathbf{d}_e(X)$ for all $n$.

Hence we must conclude that $\mu_e\restr_{t}[n_{\sigma}]\un$,
which means that clause (a) in the definition of $D_t$ with respect to 
$\sigma=X\restr_t$ does not apply.
If clause (b)  applies for  $\sigma=X\restr_t$ in the definition of $D_t$, then since $X\not\in \dbra{D_t}$
we have $\LL(X\restr_t)\neq \LL(X\restr_{t'})$ for some $t'>t$, which
contradicts clause (i) of the definition of $n_0$.
Hence we must conclude that clause (c) in the definition of $D_t$ with respect to
 $\sigma=X\restr_t$ applies.
Now consider $n^{\ast}_{\sigma}$ as in 
Definition \ref{DSuVREqmbv} and fix $\tau=X\restr_{n^{\ast}_{\sigma}}$.
Note that by the choice of $n_0$ we have $\LL(X\restr_t)=\LL(X\restr_{n^{\ast}_{\sigma}})$
so $e_{\tau}=e$.
If $\mu_e\restr_{n^{\ast}_{\sigma}}[n_{\tau}]\de$
then $\mu_e\restr_{n^{\ast}_{\sigma}}[h(n^{\ast}_{\sigma})]\de$
since $n_{\tau}<f(t)\leq h(t)<h(n^{\ast}_{\sigma})$. In this case,
by clause (ii) in the definition of $n_0$, we have $\mathbf{d}_e(X)>e$
which contradicts clause (i) of the definition of $n_0$.
Hence we can assume that 
$\mu_e\restr_{n^{\ast}_{\sigma}}[n_{\tau}]\un$, and consider the enumeration
of extensions of $\tau=X\restr_{n^{\ast}_{\sigma}}$ that occurs in $D_t$.
Since $X\not\in \dbra{D_t}$, by clause (c) of the definition of $D_t$
we have that $\LL(X\restr_{n^{\ast}_{\sigma}})\neq \LL(X\restr_{s})$ for some
$s\in (n^{\ast}_{\sigma}, n_{\tau}$. However this contradicts clause (i) of the definition of $n_0$.
From this final contradiction we can conclude that 
$\LL$ does fail for $\lambda$ on $X\in\CC^{\ast}$, which concludes our proof. 
\end{proof}

Putting everything together, we show that $f$ dominates every computable function $h$
or fails for a computable measure $\lambda$ 
on a set of $\lambda$-measure $>1/2$. Given a computable $h$ we
may assume that $h$ is increasing and consider $\lambda, \CC$ from  Lemma \ref{ADVYKUBE4x}
and the sets $(D_t)$ from Definition \ref{DSuVREqmbv}.
Considering 
$C^{\ast}:=\CC-\cup_t\dbra{D_i}$,
by Lemma \ref{9wB9Cd69sQ} we have that $\lambda(C^{\ast})>1/2$
and by Lemma \ref{pZ2ozweghy}
 we have that if $f$ does not dominate $h$ then  it fails on every real in  $C^{\ast}$.
 This concludes the proof
of Theorem \ref{UejsgfHHK}.

\section{Concluding remarks and directions for further research}\label{of5zJimRlx}
Our main result was that a learner $\LL$ can $EX$-succeed on all computable (continuous)  
measures if and only if it computes a function that dominates all computable functions.
The harder part of this equivalence was to show that a learner with the above learning power
computes a function that dominates all computable functions.
In fact, our argument in Section \ref{3hp9XGohn} shows a stronger result, with respect to
a learning notion that was also considered by Bienvenu, Monin and Shen \cite{Bienvenu2014}.
\begin{defi}
A learner $\LL$, $EX$-succeeds on a measure $\mu$ with probability $q>0$ 
if for a set of reals $X$ of $\mu$-measure at least $q$, 
the limit of $\LL(X\restr_n)$ as $n\to\infty$ exists and is an index of a 
computable measure $\nu$ such that
$X$ is $\nu$-random.
\end{defi}

In the argument of  Section \ref{3hp9XGohn} we can clearly make the $\lambda$-measure
of $\cup_i D_i$ as small as we like, which means that $\lambda{\CC^{\ast}}$ can be made
as close to 1 as we might require. This means that we have actually proved the following, stronger
statement.

\begin{thm}
If a learner EX-succeeds on all computable (continuous) measures with fixed 
positive probability $q>0$, then it computes a function which dominates all computable functions.
\end{thm}

A variation of EX-learning with oracles from the literature in algorithmic learning
is when we only allow finitely many queries to the oracle when trying to guess a suitable
measure for any specific $X$. This notion is often called EX[$\ast$]-learning and
in Fortnow et.al.\ \cite{Fortnow94extremesin,FORTNOW1994231} it was shown that an oracle
$A$ can  EX[$\ast$]-learn all computable functions if and only if 
$A\oplus \emptyset'\geq_T \emptyset''$. This notion has a direct analogue
in the context of learning measures, so we use the same notation. 
In collaboration with Fang Nan, we have proved the following.

\begin{thm}[with Fang Nan]\label{EwyLs2hH2}
An oracle $A$ can EX[$\ast$]-learn all (continuous) computable measures if and only if
$A\oplus \emptyset'\geq_T \emptyset''$. 
\end{thm}

In the classic setting, there is no
concrete characterization of the oracles that BC-learn all computable functions
(recall the notion of BC-learning from the discussion in the 
end of Section \ref{A4yxNNSDkW}).\footnote{However note that 
by Stephan and Kummer \cite{Kummer:1993:SDI:168304.168319,KUMMER1996214},
in the \ce sets
an oracle can BC-learn all computable functions if and only if it can EX-learn all
computable functions.} Hence for behavioral learning, we propose the following conjecture.

\begin{conjecture}
An oracle $A$ can BC-learn all computable functions if and only if
it can BC-learn all computable (continuous) measures. 
\end{conjecture}

One can also study the oracles that are useless for learning, in the sense that
any collection of computable measures that are learned with queries to $A$
can also be learned by a computable learner. In the classic learning setting
these oracles were characterized by 
Slaman and Solovay \cite{SlamanSol:1991} and Gasarch and Pleszkoch
\cite{Gasarch1989214}
as the 1-generic sets which are computable from
the halting problem.

Finally we wish to suggest that one can study restricted classes of computable measures
and get similar results. For example, Bienvenu, Monin and Shen \cite{Bienvenu2014}
stated and proved their theorem in terms of a general computable metric space of measures,
using the framework developed by G\'{a}cs \cite{Gacs:2005:UTA}. As a result, they where
able to draw more general conclusions, such that the fact that there is no computable learner
which EX or BC learns all computable Bernoulli measures.
A similar approach may be used in order to obtain generalizations of our Theorem
\ref{UejsgfHHK}.


\end{document}